\documentclass{article}

\usepackage{microtype}
\usepackage{graphicx}
\usepackage{subfigure}
\usepackage{booktabs} 
\usepackage{amsmath}
\usepackage{bm}
\usepackage{proof}
\usepackage{amsthm}
\usepackage{mathtools}
\usepackage{nicefrac}
\usepackage[noend]{algorithmic}
\usepackage{amssymb}

\newtheorem{theorem}{Theorem}
\newtheorem{definition}{Definition}

\DeclareMathOperator*{\argmin}{argmin}
\renewcommand{\algorithmiccomment}[1]{\bgroup\hfill\scriptsize#1\egroup}

\usepackage{hyperref}

\usepackage[accepted]{icml2018_style/icml2018}

\icmltitlerunning{Pseudo-task Augmentation}

\begin{document}

\twocolumn[

\icmltitle{Pseudo-task Augmentation: From Deep Multitask\\ Learning to Intratask Sharing---and Back}




\begin{icmlauthorlist}
\icmlauthor{Elliot Meyerson}{texas,sentient}
\icmlauthor{Risto Miikkulainen}{texas,sentient}
\end{icmlauthorlist}

\icmlaffiliation{texas}{The University of Texas at Austin}
\icmlaffiliation{sentient}{Sentient Technologies, Inc.}

\icmlcorrespondingauthor{Elliot Meyerson}{ekm@cs.utexas.edu}

\icmlkeywords{Deep Learning, Multitask Learning, Pseudo-task Augmentation, Multiple Models, Model Search}

\vskip 0.3in
]



\printAffiliationsAndNotice{}  

\begin{abstract}

Deep multitask learning boosts performance by sharing learned structure across related tasks.
This paper adapts ideas from deep multitask learning to the setting where only a single task is available.
The method is formalized as \emph{pseudo-task augmentation}, in which models are trained with multiple decoders for each task.
Pseudo-tasks simulate the effect of training towards closely-related tasks drawn from the same universe.
In a suite of experiments, pseudo-task augmentation improves performance on single-task learning problems.
When combined with multitask learning, further improvements are achieved, including state-of-the-art performance on the CelebA dataset, showing that pseudo-task augmentation and multitask learning have complementary value.
All in all, pseudo-task augmentation is a broadly applicable and efficient way to boost performance in deep learning systems.

\end{abstract}

\section{Introduction}
\label{sec:introduction}

Multitask learning (MTL) \cite{Caruana:1998} improves performance by leveraging relationships between distinct learning problems.
In recent years, MTL has been extended to deep learning, in which it has improved performance in applications such as vision \cite{Zhang:2014, Bilen:2016, Misra:2016, Rudd:2016, Lu:2016, Rebuffi:2017, Yang:2017}, natural language \cite{Collobert:2008, Dong:2015,  Liu:2015, Luong:2016, Hashimoto:2017}, speech \cite{Huang:2013, Seltzer:2013, Huang:2015, Wu:2015}, reinforcement learning \cite{Devin:2016, Jaderberg:2016, Teh:2017}, and even seemingly unrelated tasks from disparate domains \cite{Kaiser:2017, Meyerson:2018}.
Deep MTL relies on training signals from multiple datasets to train deep structure that is shared across tasks.
Since the shared structure must support solving multiple problems, it is inherently more general, which leads to better generalization to holdout data.

This paper adapts ideas from deep MTL to the single-task learning (STL) case, i.e., when only a single task is available for training.
The method is formalized as \emph{pseudo-task augmentation} (PTA), in which a single task has multiple distinct decoders projecting the output of the shared structure to task predictions.
By training the shared structure to \emph{solve the same problem in multiple ways}, PTA simulates the effect of training towards distinct but closely-related tasks drawn from the same universe.
Theoretical justification shows how training dynamics with multiple pseudo-tasks strictly subsumes training with just one, and a class of algorithms is introduced for controlling pseudo-tasks in practice.

In an array of experiments, PTA is shown to significantly improve performance in single-task settings.
Although different variants of PTA traverse the space of pseudo-tasks in qualitatively different ways, they all demonstrate substantial gains.
Experiments also show that when PTA is combined with MTL, further improvements are achieved, including state-of-the-art performance on the CelebA dataset.
In other words, although PTA can be seen as a base case of MTL, PTA and MTL have complementary value in learning more generalizable models.
The conclusion is that pseudo-task augmentation is an efficient, reliable, and broadly applicable method for boosting performance in deep learning systems. 

The remainder of the paper is organized as follows: Section~\ref{sec:training_multiple_deep_models} covers background on deep learning methods that train multiple models; Section~\ref{sec:pseudotask_augmentation} introduces the pseudo-task augmentation framework and practical implementations; Section~\ref{sec:experiments} describes experimental setups and results; Sections~\ref{sec:discussion} and \ref{sec:conclusion} discuss future work and overall implications.

\section{Training Multiple Deep Models}
\label{sec:training_multiple_deep_models}

There is a broad range of methods that exploit synergies across multiple deep models.
This section reviews these methods by classifying them into three types:
(1) methods that jointly train a model for multiple tasks;
(2) methods that train multiple models separately for a single task;
and (3) methods that jointly train multiple models for a single task.
This review motivates the development of methods in (3) that unify the advantages of (1) and (2).

\subsection{Joint training of models for multiple tasks}
\label{subsec:multitask_learning}

There are many real-world scenarios where harnessing data from multiple related tasks can improve overall performance.
In general, there are $T$ tasks $\{\{\bm{x}_{ti}, \bm{y}_{ti}\}_{i=1}^{N_t}\}_{t=1}^T$, where $N_t$ is the number of samples for the $t$th task.
Note that it is possible that for $t_1 \neq t_2$, $N_{t_1} \neq N_{t_2}$, $\dim(\bm{x}_{t_1}) \neq \dim(\bm{x}_{t_2})$, and/or $\dim(\bm{y}_{t_1}) \neq \dim(\bm{y}_{t_2})$.
The only requirement for multitask learning to be useful is that there is some amount of information shared across tasks, and, in theory, this is always the case \cite{Mahmud:2008, Mahmud:2009}.

Joint training of neural network models for multiple tasks was proposed decades ago \cite{Caruana:1998}.
Modern approaches have extended this early work to deep learning.
Though more sophisticated methods now exist, the most common approach is still based on the original work, in which a joint model is decomposed into an underlying model $\mathcal{F}$ (parameterized by $\theta_{\mathcal{F}}$) that is shared across all tasks, and task-specific decoders $\mathcal{D}_t$ (parameterized by $\theta_{\mathcal{D}_t}$) for each task.
The model for the $t$th task is then defined as
\begin{equation}
\label{eq:multitask_model}
\hat{\bm{y}}_{ti} = \mathcal{D}_t(\mathcal{F}(\bm{x}_{ti}; \theta_{\mathcal{F}}); \theta_{\mathcal{D}_t}) \, .
\end{equation}
Given a fixed model architecture for all $\mathcal{D}_t$ and $\mathcal{F}$, the joint model is completely defined by the parameters $\theta = (\{\theta_{\mathcal{D}_t}\}_{t=1}^T, \theta_{\mathcal{F}})$.
To maximize overall performance, the goal is to find optimal parameters $\theta^*$ such that
\begin{equation}
\label{eq:multitask_loss}
\theta^* = \argmin_{\theta}\frac{1}{T}\sum_{t=1}^{T}\frac{1}{N_t}\sum_{i=1}^{N_t}\mathcal{L}(\bm{y}_{ti}, \hat{\bm{y}}_{ti})
\end{equation}
for a suitable sample-wise loss function $\mathcal{L}$, e.g., mean squared error or cross-entropy loss.
More sophisticated deep MTL approaches can be characterized by the design decision of \emph{how} learned structure is shared across tasks.
For example, some methods supervise different tasks at different depths of the shared structure \cite{Zhang:2016, Hashimoto:2017, Toshniwal:2017}; other methods duplicate the shared structure into columns and define mechanisms for sharing information across columns \cite{Jou:2016,Misra:2016,Long:2017,Yang:2017}.
More detailed characterizations of deep MTL methods can be found in previous work \cite{ Ruder:2017, Meyerson:2018}.

MTL has also been explored extensively outside of deep learning. 
Many such techniques take a similar approach of having shared structure with a separate linear decoder for each task, while enforcing regularization of shared convex structure \cite{Evgeniou:2004,Argyriou:2008,Kang:2011,Kumar:2012}.
Overall, by requiring models to fit multiple real world datasets simultaneously, MTL is a promising approach to learning more realistic, and thus more generalizable, models.

\subsection{Separate training of multiple models for STL}
\label{subsec:single_task_separate_models}

How to construct and train deep a neural network effectively is an open-ended design problem even in the case of a single task.
A range of methods have been developed that aim at overcoming this problem by training multiple models separately for a single task.
One class of methods searches for optimal fixed designs, e.g., by automatically optimizing learning hyperparameters \cite{Bergstra:2011,Snoek:2012} or more open-ended network topologies \cite{Miikkulainen:2017, Real:2017, Zoph:2017}.
The multiple models synergize by providing complementary information about different areas of the search space, and, over time, the results of past models can be used to generate better models.
Population-based training takes this one step further, by copying the weights of successful models to new models \cite{Jaderberg:2017}.
This weight-copying is similar to methods that transfer learned behavior across a sequence of pre-defined architectures \cite{Hinton:2015, Chen:2016, Wei:2016}.
The synergy of multiple models can also be exploited via ensembling \cite{Dietterich:2000}.
Overall, the widespread success of the above methods have shown the value of training multiple models separately, both sequentially and in parallel.

\subsection{Joint training of multiple models for STL}
\label{subsec:single_task_joint_model}

Some existing methods can be viewed as jointly training multiple models for a single task.
For instance, to improve training of deep models, deep supervision includes loss layers at multiple depths \cite{Lee:2015}.
As a by-product, this approach yields a distinct model for the task at each such depth, though only the deepest model is ever evaluated.
As another example, dropout \cite{Srivastava:2014}, and pseudo-ensembles more generally \cite{Bachman:2014}, can be seen as implicitly training many relatively weak models that are combined during evaluation.
Also, PathNet \cite{Fernando:2017} jointly trains multiple networks induced by various paths through a set of shared modules.
However, the goal is not to improve single task performance, but discover structure that can be effectively reused by future tasks.
Although these existing methods jointly train multiple models for a single task, they do not perform joint training in the MTL sense.
Ideally, the benefits of the methods in Sections~\ref{subsec:multitask_learning} and \ref{subsec:single_task_separate_models} could be combined, yielding methods that train multiple models that share underlying parameters \emph{and} sample complementary high-performing areas of the model space.
This paper takes first steps in that direction, showing that such methods are indeed promising.
The specific approach, PTA, is introduced in the next section.

\section{Pseudo-task Augmentation (PTA)}
\label{sec:pseudotask_augmentation}

This section introduces the PTA method.
First, the classical deep MTL approach is extended to the case of multiple decoders per task.
Then, the concept of a pseudo-task is introduced, and increased training dynamics under multiple pseudo-tasks is demonstrated.
Finally, practical methods for controlling pseudo-tasks during training are described, which will be compared empirically in Section~\ref{sec:experiments}.

\subsection{A Classical Approach}
\label{subsec:classical}

The most common approach to deep MTL is still the ``classical'' approach (Eq.~\ref{eq:multitask_model}), in which all layers are shared across all tasks up to a high level, after which each task learns a distinct decoder that maps high-level points to its task-specific output space \cite{Caruana:1998, Ranjan:2016, Lu:2016}.
Even when more sophisticated methods are developed, the classical approach is often used as a baseline for comparison.
The classical approach is also computationally efficient, in that the only additional parameters beyond a single task model are in the additional decoders.
Thus, when applying ideas from deep MTL to single-task multi-model learning, the classical approach is a natural starting point.

Consider again the case where there are $T$ distinct \emph{true} tasks, but now let there be $D$ decoders for each task.
Then, the model for the $d$th decoder of the $t$th task is given by
\begin{equation}
\label{eq:pseudotask_model}
\hat{\bm{y}}_{tdi} = \mathcal{D}_{td}(\mathcal{F}(\bm{x}_{ti}; \theta_{\mathcal{F}}); \theta_{\mathcal{D}_{td}}) \, ,
\end{equation}
and the overall loss for the joint model from Eq.~\ref{eq:multitask_loss} becomes
\begin{equation}
\label{eq:pseudotask_loss}
\theta^* = \argmin_{\theta}\frac{1}{TD}\sum_{t=1}^{T}\frac{1}{N_t}\sum_{i=1}^{N_t}\sum_{d=1}^D\mathcal{L}(\bm{y}_{ti}, \hat{\bm{y}}_{tdi}) \, ,
\end{equation}
where $\theta = (\{\{\theta_{\mathcal{D}_{td}}\}_{d=1}^D\}_{t=1}^T, \theta_\mathcal{F})$.
In the same way as the classical approach to MTL encourages $\mathcal{F}$ to be more general and robust by requiring it to support multiple tasks, here $\mathcal{F}$ is required to support \emph{solving the same task in multiple ways}.
A visualization of a resulting joint model is shown in Figure~\ref{fig:pseudotask}.
\begin{figure}
\centering
\includegraphics[width=\linewidth]{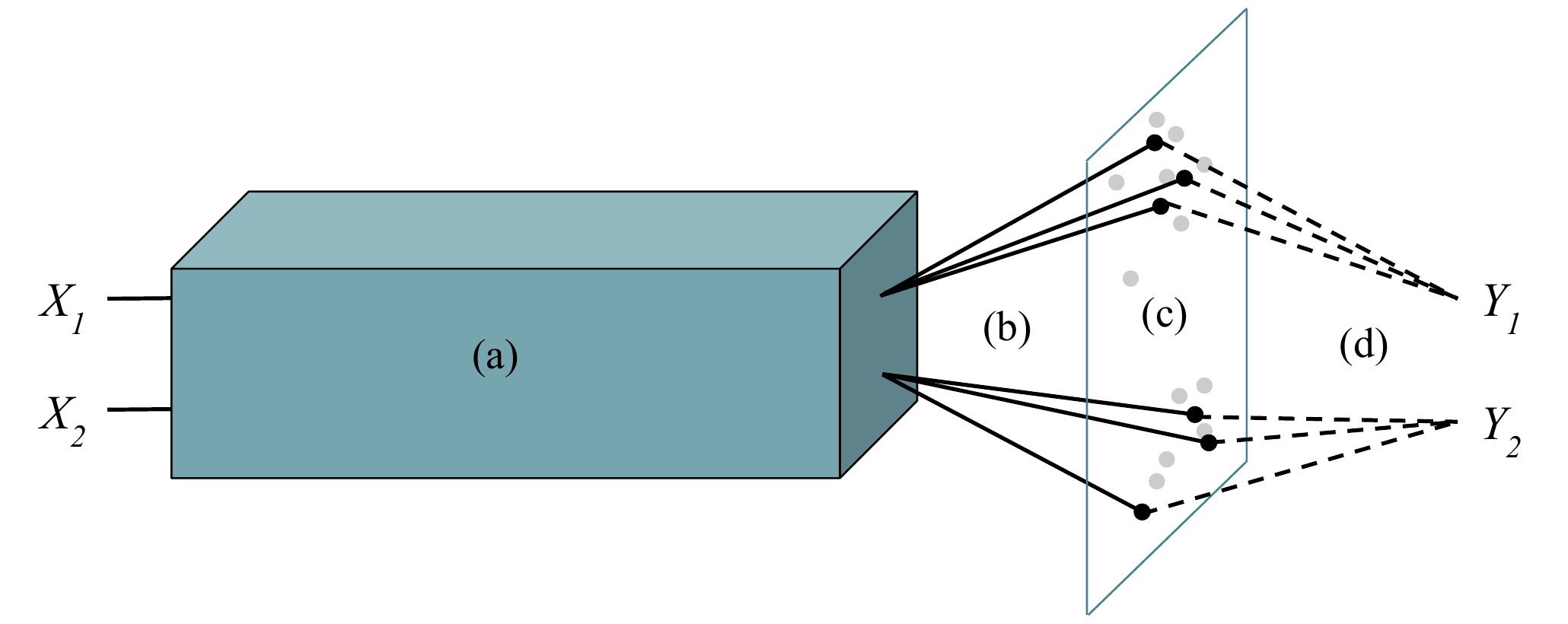}
\caption{\label{fig:pseudotask} General setup for pseudo-task augmentation with two tasks.
(a) \emph{Underlying model.} All task inputs are embedded through an underlying model that is completely shared;
(b) \emph{Multiple decoders.} Each task has multiple decoders (solid black lines) each projecting the embedding to a distinct classification layer;
(c) \emph{Parallel traversal of model space.} The underlying model coupled with a decoder defines a task model.
Task models populate a model space, with current models shown as black dots and previous models shown as gray dots;
(d) \emph{Multiple loss signals.} Each current task model receives a distinct loss to compute its distinct gradient.
A task coupled with a decoder and its parameters defines a pseudo-task for the underlying model.
}
\end{figure}
A theme in MTL is that models for related tasks will have similar decoders, as implemented by explicit regularization \cite{Evgeniou:2004, Kumar:2012, Long:2017, Yang:2017}.
Similarly, in Eq.~\ref{eq:pseudotask_loss}, through training, two decoders for the same task will instantiate similar models, and, as long as they do not converge completely to equality, they will simulate the effect of training with multiple closely-related tasks.

Notice that the innermost summation in Eq.~\ref{eq:pseudotask_loss} is over decoders.
This calculation is computationally efficient: because each decoder for a given task takes the same input, $\mathcal{F}(\bm{x}_{ti})$ (usually the most expensive part of the model) need only be computed once per sample (and only once over all tasks if all tasks share $\bm{x}_{ti}$).
However, when evaluating the performance of a model, since each decoder induces a distinct model for a task, what matters is not the average over decoders, but the best performing decoder for each task, i.e.,
\begin{equation}
\label{eq:pseudotask_evaluation}
\theta^*_{\text{eval}} = \argmin_{\theta}\frac{1}{T}\sum_{t=1}^{T}\frac{1}{N_t}\argmin_{d \in 1..D}\sum_{i=1}^{N_t}\mathcal{L}(\bm{y}_{ti}, \hat{\bm{y}}_{tdi}).
\end{equation}
Eq.~\ref{eq:pseudotask_loss} is used in training because it is smoother; Equation~\ref{eq:pseudotask_evaluation} is used for model validation, and to select the best performing decoder for each task from the final joint model. This decoder is then applied to future data, e.g., a holdout set.
Once the models are trained, in principle they form a set of distinct and equally powerful models for each task. 
It may therefore be tempting to ensemble them for evaluation, i.e.,
\begin{equation}
\label{eq:pseudotask_ensemble}
\theta^*_{\text{ens}} = \argmin_{\theta}\frac{1}{T}\sum_{t=1}^{T}\frac{1}{N_t}\sum_{i=1}^{N_t}\mathcal{L}\bigg(\bm{y}_{ti}, \frac{1}{D}\sum_{d=1}^D\hat{\bm{y}}_{tdi}\bigg).
\end{equation}
However, with linear decoders, \emph{training} with Eq.~\ref{eq:pseudotask_ensemble} is equivalent to training with a single decoder for each task, while training with Eq.~\ref{eq:pseudotask_loss} with multiple decoders yields more expressive training dynamics.
These ideas are developed more fully in the next section.

\subsection{Pseudo-tasks}
\label{subsec:pseudotasks}

Following the intuition that training $\mathcal{F}$ with multiple decoders amounts to solving the task in multiple ways, each ``way'' is defined by a \emph{pseudo-task}
\begin{equation}
\label{eq:pseudotask}
(\mathcal{D}_{td}, \mathcal{\theta}_{td}, \{\bm{x}_{ti}, \bm{y}_{ti}\}_{i=1}^{N_t})
\end{equation}
of the \emph{true} underlying task $\{\bm{x}_{ti}, \bm{y}_{ti}\}_{i=1}^{N_t}$. 
It is termed a pseudo-task because it derives from a true task, but has no fixed labels.
That is, for any fixed $\theta_{td}$, there are potentially many optimal outputs for $\mathcal{F}(\bm{x}_{ti})$.
When $D > 1$, training $\mathcal{F}$ amounts to training each task with multiple pseudo-tasks for each task at each gradient update step.
This process is the essence of PTA.

As a first step, this paper considers linear decoders, i.e. each $\mathcal{D}_{td}$ consists of a single dense layer of weights (any following nonlinearity can be considered part of the loss function).
Prior work has assumed that models for closely-related tasks differ only by a linear transformation \cite{Evgeniou:2004, Kang:2011, Argyriou:2008}.
Similarly, with linear decoders, distinct pseudo-tasks for the same task simulate multiple closely-related tasks.
When $\theta_{td}$ are considered fixed, the learning problem (Eq.~\ref{eq:pseudotask_loss}) reduces to
\begin{equation}
\label{eq:fixed_pseudotask_loss}
\theta_{\mathcal{F}}^* = \argmin_{\theta_{\mathcal{F}}}\frac{1}{TD}\sum_{t=1}^{T}\frac{1}{N_t}\sum_{i=1}^{N_t}\sum_{d=1}^D\mathcal{L}(\bm{y}_{ti}, \hat{\bm{y}}_{tdi}) \, .
\end{equation}
In other words, although the overall goal is to learn models for $T$ tasks, $\mathcal{F}$ is at each step optimized towards $DT$ pseudo-tasks.
Thus, training with multiple decoders may yield positive effects similar to training with multiple true tasks.

After training, the best model for a given task is selected from the final joint model, and used as the final model for that task (Eq.~\ref{eq:pseudotask_evaluation}).
Of course, using multiple decoders with identical architectures for a single task does not make the final learned predictive models more expressive.
It is therefore natural to ask whether including additional decoders has any fundamental effect on learning dynamics.
It turns out that even in the case of linear decoders, the training dynamics of using multiple pseudo-tasks strictly subsumes using just one.

\begin{definition}[Pseudo-task Simulation]
A set of pseudo-tasks $S_1$ \emph{simulates} another $S_2$ on $\mathcal{F}$ if for all $\theta_\mathcal{F}$ the gradient update to $\theta_\mathcal{F}$ when trained with $S_1$ is equal to that with $S_2$.
\end{definition}

\begin{theorem}[Augmented Training Dynamics]
\label{thm:augmented}
There exist differentiable functions $\mathcal{F}$ and sets of pseudo-tasks of a single task that cannot be simulated by a single pseudo-task of that task, even when all decoders are linear.
\end{theorem}

\begin{proof}
Consider a task with a single sample $(\bm{x}, y)$, where $y$ is a scalar. 
Suppose $\mathcal{L}$ (from Eq.~\ref{eq:fixed_pseudotask_loss}) computes mean squared error, $\mathcal{F}$ has output dimension $M$, and all decoders are linear, with bias terms omitted for clarity.
$\mathcal{D}_d$ is then completely specified by the vector $\bm{w}_d = \langle w_d^1, w_d^2, ..., w_d^M \rangle^\top$.
Suppose parameter updates are performed by gradient descent.
The update rule for $\theta_{\mathcal{F}}$ with fixed decoders $\{\mathcal{D}_d\}_{d=1}^D$ and learning rate $\alpha$ is then given by
\begin{equation}
\label{eq:update_rule}
\theta_{\mathcal{F}} \coloneqq \theta_{\mathcal{F}} - \alpha \sum_{d=1}^D\nabla_{\mathcal{F}}\big(y - \bm{w}_d^\top\mathcal{F}(\bm{x}; \theta_{\mathcal{F}})\big)^2 \, .
\end{equation}
For a single fixed decoder to yield equivalent behavior, it must have equivalent update steps.
The goal then is to choose $(\bm{x}, y)$, $\mathcal{F}$, $\{\theta_k\}_{k=1}^K$, $\{\bm{w}_d\}_{d=1}^D$, and $\alpha > 0$, such that there are no $\bm{w}_o$, $\gamma > 0$, for which $\forall \, k$
\begin{multline}
\label{eq:condition_to_contradict}
\alpha \sum_{d=1}^D\nabla_{\mathcal{F}}(y - \bm{w}_d^\top\mathcal{F}(\bm{x}; \theta_{k}))^2 =\\
\gamma\nabla_{\mathcal{F}}(y - \bm{w}_o^\top\mathcal{F}(\bm{x}; \theta_{k}))^2\\
\implies 
\alpha \sum_{d=1}^D2(y - \bm{w}_d^\top\mathcal{F}(\bm{x}; \theta_{k}))\bm{w}_d^\top J_{\mathcal{F}}(\bm{x}; \theta_{k}) =\\
2\gamma(y - \bm{w}_o^\top\mathcal{F}(\bm{x}; \theta_{k}))\bm{w}_o^\top J_\mathcal{F}(\bm{x}; \theta_{k}) \, ,
\end{multline}
where $J_\mathcal{F}$ is the Jacobian of $\mathcal{F}$.
By choosing $\mathcal{F}$ and $\{\theta_k\}_{k=1}^K$ so that all $J_\mathcal{F}(\bm{x}; \theta_{k})$ have full row rank, Eq.~\ref{eq:condition_to_contradict} reduces to
\begin{multline}
\label{eq:simplified_condition}
\alpha \sum_{d=1}^D(y - \bm{w}_d^\top\mathcal{F}(\bm{x}; \theta_{k}))w_d^i =\\
\gamma (y - \bm{w}_o^\top\mathcal{F}(\bm{x}; \theta_{k}))w_o^i \ \ \ \forall \ i \in 1..M.
\end{multline}
Choosing $\mathcal{F}$, $\{\theta_k\}_{k=1}^K$, $\{\bm{w}_d\}_{d=1}^D$, and $\alpha > 0$ such that the left hand side of Eq.~\ref{eq:simplified_condition} is never zero, we can safely write
\begin{multline}
\frac{\sum_{d=1}^D(y - \bm{w}_d^\top\mathcal{F}(\bm{x}; \theta_{k}))w_d^i}{(y - \bm{w}_o^\top\mathcal{F}(\bm{x}; \theta_{k}))w_o^i}=\\
\frac{\sum_{d=1}^D(y - \bm{w}_d^\top\mathcal{F}(\bm{x}; \theta_{k}))w_d^j}{(y - \bm{w}_o^\top\mathcal{F}(\bm{x}; \theta_{k}))w_o^j} \ \ \ \forall \  (i,j) \\
\implies
\frac{\sum_{d=1}^D(y - \bm{w}_d^\top\mathcal{F}(\bm{x}; \theta_{k}))w_d^i}{\sum_{d=1}^D(y - \bm{w}_d^\top\mathcal{F}(\bm{x}; \theta_{k}))w_d^j} = \frac{w_o^i}{w_o^j} \, .
\end{multline}
Then, since $\bm{w}_o$ is fixed, it suffices to find $\mathcal{F}(\bm{x}; \theta_1)$, $\mathcal{F}(\bm{x}; \theta_2)$ such that for some $(i, j)$
\begin{multline}
\frac{\sum_{d=1}^D(y - \bm{w}_d^\top\mathcal{F}(\bm{x}; \theta_{1}))w_d^i}{\sum_{d=1}^D(y - \bm{w}_d^\top\mathcal{F}(\bm{x}; \theta_{1}))w_d^j} \neq \\ \frac{\sum_{d=1}^D(y - \bm{w}_d^\top\mathcal{F}(\bm{x}; \theta_{2}))w_d^i}{\sum_{d=1}^D(y - \bm{w}_d^\top\mathcal{F}(\bm{x}; \theta_{2}))w_d^j} \, .
\end{multline}
For instance, with $D = 2$, choosing $y = 1$, $w_1 = \langle 2, 3 \rangle^\top$, $w_2 = \langle 4, 5 \rangle^\top$,  $\mathcal{F}(\bm{x}; \theta_1) = \langle 6, 7 \rangle^\top$,  and $\mathcal{F}(\bm{x}; \theta_1) = \langle 8, 9 \rangle^\top$ satisfies the inequality.
Note $\mathcal{F}(\bm{x}; \theta_1)$ and $\mathcal{F}(\bm{x}; \theta_2)$ can be chosen arbitrarily since $\mathcal{F}$ is only required to be differentiable, e.g., implemented by a neural network.
\end{proof}

Showing that a single pseudo-task can be simulated by $D$ pseudo-tasks for any $D > 1$ is more direct:
For any $\bm{w}_o$ and $\gamma$, choose $\bm{w}_d = \bm{w}_o \ \forall \ d \in 1..D$ and $\alpha = \nicefrac{\gamma}{D}$.
Further extensions to tasks with more samples, higher dimensional outputs, and cross-entropy loss are straightforward.
Note that this result is related to work on the dynamics of deep linear models \cite{Saxe:2013}, in that adding additional linear structure complexifies training dynamics.
However, training an ensemble directly, i.e., via Eq.~\ref{eq:pseudotask_ensemble}, does \emph{not} yield augmented training dynamics, since
\begin{multline}
\label{eq:ensemble_collapse}
\frac{1}{D}\sum_{d=1}^D\hat{\bm{y}}_{tdi} =\frac{1}{D}\sum_{d=1}^D\bm{W}_{td}^\top\mathcal{F}(\bm{x}_{ti}; \theta_{\mathcal{F}}) \\
\implies \bm{W}_{to}^\top = \frac{1}{D}\sum_{d=1}^D\bm{W}_{td}^\top \text{ and } \beta = \alpha \, .
\end{multline}

Now that we know that training with additional pseudo-tasks yields augmented training dynamics that may be exploited, the question is how to take advantage of these dynamics in practice.
The next section introduces methods to address this question.

\subsection{Control of Multiple Pseudo-task Trajectories}

Given linear decoders, the primary goal is to optimize $\mathcal{F}$; if an optimal $\mathcal{F}$ were found, optimal decoders for each task could be derived analytically.
So, given multiple linear decoders for each task, how should their induced pseudo-tasks be controlled to maximize the benefit to $\mathcal{F}$?
For one, their weights $\{\bm{W}_{td}\}_{d=1}^D$ must not all be equal, otherwise we would have $\bm{W}_{to} = \bm{W}_{t1}$ and $\gamma = D\alpha$ in the proof of Theorem~\ref{thm:augmented}.
Following Eq.~\ref{eq:pseudotask_loss}, decoders can be trained jointly with $\mathcal{F}$ via gradient-based methods, so that they learn to work well with $\mathcal{F}$.
Through optimization, a trained decoder induces a trajectory of pseudo-tasks.
Going beyond this \emph{implicit} control, Algorithm~\ref{alg:pta_training} gives a high-level framework for applying \emph{explicit} control to pseudo-task trajectories.
\begin{algorithm}
\caption{\label{alg:pta_training} PTA Training Framework}
\small
\begin{algorithmic}[1]
\STATE Given $T$ tasks $\{\{\bm{x}_{ti}, \bm{y}_{ti}\}_{i=1}^{N_t}\}_{t=1}^T$, and $D$ decoders per task
\STATE $\{\{\theta_{\mathcal{D}_{td}}\}_{d=1}^D\}_{t=1}^T \leftarrow \textit{DecInitialize}()$
\STATE Initialize $\theta_{\mathcal{F}}$
\STATE Initialize decoder costs $c_{td} \leftarrow \infty \ \forall \ (t, d)$
\WHILE{not done training}
	\FOR[$\vartriangleright M$ is meta-iteration length]{$m=1$ \TO $M$}
		\STATE Update $\theta_{\mathcal{F}}$ and $\theta_{\mathcal{D}_{td}}$ via a joint gradient step. \label{ln:gradient_step}
	\ENDFOR
	\FOR{$t=1$ \TO $T$}
		\FOR{$d=1$ \TO $D$}
			\STATE $c_{td} \leftarrow \text{evaluate}(\theta_{\mathcal{D}_{td}}, \theta_\mathcal{F}, t)$ \COMMENT{$\vartriangleright$ e.g., get validation error}
		\ENDFOR
		\FOR{$d=1$ \TO $D$}
			\STATE $\theta_{\mathcal{D}_{td}} \leftarrow \textit{DecUpdate}\big(d, \{\theta_{\mathcal{D}_{td_o}}, c_{td_o}\}_{d_o = 1}^D\big)$
		\ENDFOR
	\ENDFOR
\ENDWHILE
\RETURN $\big(\{\{\theta_{\mathcal{D}_{td}}\}_{d=1}^D\}_{t=1}^T, \theta_{\mathcal{F}}\big)$
\end{algorithmic}
\end{algorithm}

An instance of the algorithm is parameterized by choices for $\textit{DecInitialize}$, which defines how decoders are initialized; and $\textit{DecUpdate}$, which defines non-gradient-based updates to decoders every $M$ gradient steps, i.e., every \emph{meta-iteration}, based on the performance of each decoder ($\textit{DecUpdate}$ defaults to no-op).
As a first step, several intuitive methods are evaluated in this paper for instantiating Algorithm~\ref{alg:pta_training}.
These methods can be used together in any combination:

\textbf{Independent Initialization (I)}
$\textit{DecInitialize}$ randomly initializes all $\theta_{\mathcal{D}_{td}}$ independently.
This is the obvious initialization method, and is assumed in all methods below.

\textbf{Freeze (F)}
$\textit{DecInitialize}$ freezes all decoder weights except $\theta_{\mathcal{D}_{t1}}$ for each task.
Frozen weights do not receive gradient updates in Line~\ref{ln:gradient_step} of Algorithm~\ref{alg:pta_training}.
Because they cannot adapt to $\mathcal{F}$, constant pseudo-task trajectories provide a stricter constraint on $\mathcal{F}$.
One decoder is left unfrozen so that the optimal model for each task can still be learned.

\textbf{Independent Dropout (D)}
$\textit{DecInitialize}$ sets up the dropout layers preceding linear decoder layers to drop out values independently for each decoder.
Thus, even when the weights of two decoders for a task are equal, their resulting gradient updates to $\mathcal{F}$ and to themselves will be different.

For the next three methods, let $c_{t}^{\text{min}} = \min(c_{t1},\ldots,c_{tD})$.

\textbf{Perturb (P)}
$\textit{DecUpdate}$ adds noise $\sim \mathcal{N}(\bm{0}, \epsilon_p\bm{I})$ to each $\theta_{\mathcal{D}_{td}}$ for all $d$ where $c_{td} \neq c_t^\text{min}$.
This method ensures that $\theta_{\mathcal{D}_{td}}$ are sufficiently distinct before each training period.

\textbf{Hyperperturb (H)}
Like \emph{Perturb}, except $\textit{DecUpdate}$ updates the hyperparameters of each decoder other than the best for each task, by adding noise $\sim \mathcal{N}(0, \epsilon_h)$.
In this paper, each decoder has only one hyperparameter: the dropout rate of any \emph{Independent Dropout} layer, because adapting dropout rates can be beneficial \cite{Ba:2013, Li:2016, Jaderberg:2017}.

\textbf{Greedy (G)}
For each task, let $\theta_t^\text{min}$ be the weights of a decoder with cost $c_t^\text{min}$.
$\textit{DecUpdate}$ updates all $\theta_{td} \coloneqq \theta_t^\text{min}$, including hyperparameters.
This biases training to explore the highest-performing areas of the pseudo-task space.
When combined with any of the previous three methods, decoder weights are still ensured to be distinct through training.

Combinations of these six methods induce an initial class of PTA training algorithms PTA-* for the case of linear decoders.
The next section evaluates eight representative combinations of these methods, i.e., PTA-I, PTA-F, \mbox{PTA-P}, PTA-D, PTA-FP, PTA-GP, PTA-GD, and PTA-HGD, in various experimental settings.
Note that H and G are related to methods that copy the weights of the entire network \cite{Jaderberg:2017}.
Also note that, in a possible future extension to the nonlinear case, the space of possible PTA control methods becomes much more broad, as will be discussed in Section~\ref{sec:discussion}.

\section{Experiments}
\label{sec:experiments}

In this section, PTA methods are evaluated and shown to excel in a range of settings: (1) single-task character recognition; (2) multitask character recognition; (3) single-task sentiment classification; and (4) multitask visual attribute classification.
All experiments are implemented using the Keras framework \cite{Chollet:2015}. For PTA-P and PTA-GP, $\epsilon_p = 0.01$; for PTA-HGD, $\epsilon_h = 0.1$ and dropout rates range from 0.2 to 0.8.
A dropout layer with dropout rate initialized to 0.5 precedes each decoder.

\subsection{Omniglot Character Recognition}
\label{subsec:omni}

This section evaluates and compares the various PTA methods on Omniglot character recognition \cite{Lake:2015}.
The Omniglot dataset consists of 50 alphabets of handwritten characters, each of which induces its own character recognition task.
Each character instance is a $105\times105$ black-and-white image, and each character has 20 instances, each drawn by a different individual.
To reduce variance and improve reproducibility of experiments, a fixed random 50/20/30\% train/validation/test split was used for each task.
(These splits will be released with the paper.)
Methods are evaluated with respect to all 50 tasks as well as a subset consisting of the first 20 tasks in a fixed random ordering of alphabets used in previous work \cite{Meyerson:2018}.
The underlying model $\mathcal{F}$ for all setups is a simple four layer convolutional network that has been shown to yield good performance on Omniglot \cite{Meyerson:2018}.
This model has four convolutional layers each with 53 filters and $3\times3$ kernels, and each followed by a $2\times2$ max-pooling layer and dropout layer with 0.5 dropout probability.
At each meta-iteration, 250 gradient updates are performed via Adam \cite{Kingma:14}; each setup is trained for 100 meta-iterations.

\subsubsection{Omniglot: Single-task Learning}
\label{subsubsec:omni_stl}

The single-task learning case is considered first.
For each of the 20 initial Omniglot tasks, the eight PTA methods were applied to the task with 2, 3, and 4 decoders.
At least three trials were run with each setup; the mean performance averaged across trials and tasks is shown in Figure~\ref{fig:omni_stl}.
\begin{figure}
\centering
\includegraphics[width=0.9\linewidth]{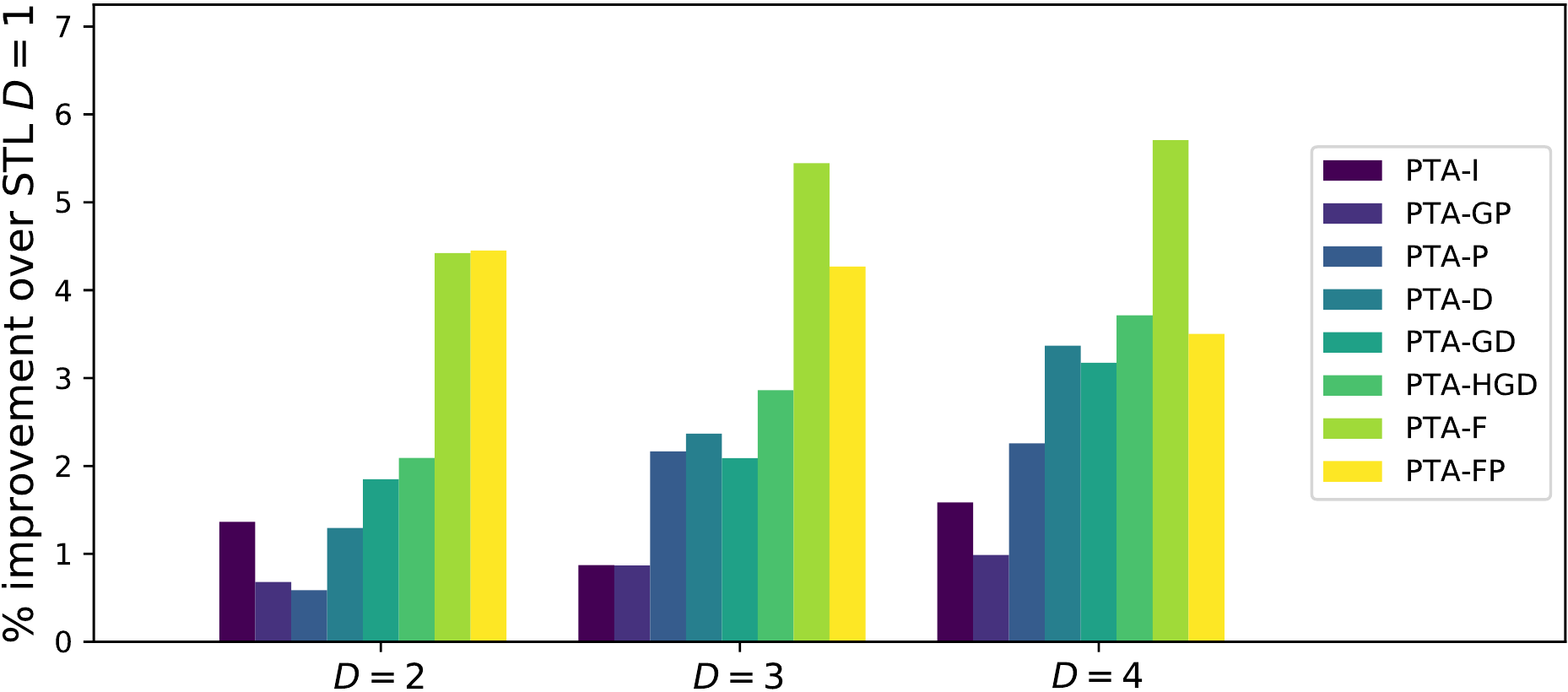}
\caption{\label{fig:omni_stl} \textbf{Omniglot single-task learning results.} For each number of decoders $D$, mean improvement (absolute \% decrease in error) over $D=1$ is plotted for each setup, averaged across all tasks.
All setups outperform the baseline.
PTA-F and PTA-FP performs best, as this problem benefits from strong regularization.
The mean improvement across all methods also increases with $D$: 1.86\% for $D=2$; 2.33\% for $D=3$; and 2.70\% for $D=4$.
}
\end{figure}
Every PTA setup outperforms the baseline, i.e., training with a single decoder.
The methods that use decoder freezing, \mbox{PTA-F} and PTA-FP, perform best, showing how this problem can benefit from strong regularization.
Notably, the mean improvement across all methods increases with $D$: 1.86\% for $D=2$; 2.33\% for $D=3$; and 2.70\% for $D=4$.
Like MTL can benefit from adding more tasks \cite{Caruana:1998, Hashimoto:2017, Jaderberg:2016}, single-task learning can benefit from adding more pseudo-tasks.

\subsubsection{Omniglot: Multitask Learning}
\label{subsubsec:omni_mtl}

Omniglot models have also been shown to benefit from MTL \cite{Maclaurin:2015, Rebuffi:2017, Yang:2017, Meyerson:2018}.
This section extends the experiments in Section~\ref{subsubsec:omni_stl} to MTL.
The setup is exactly the same, except now the underlying convolutional model is fully shared across all tasks for each method.
The results are shown in Figure~\ref{fig:omni_mtl}.
\begin{figure}
\centering
\includegraphics[width=0.9\linewidth]{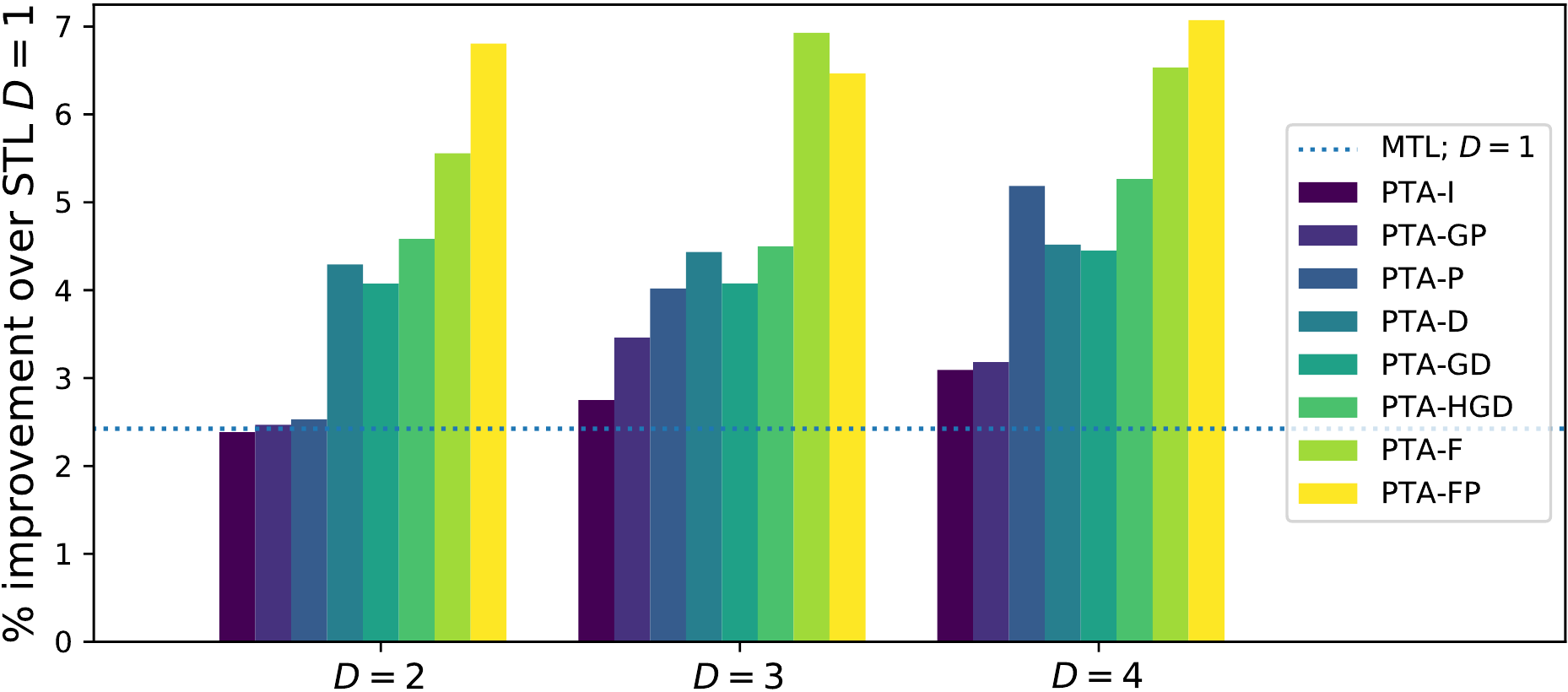}
\caption{\label{fig:omni_mtl} \textbf{Omniglot multitask learning results.} For each number of decoders $D$, mean improvement (absolute \% decrease in error) across all tasks is plotted over STL with $D=1$.
All setups outperform the STL baseline, and all except PTA-I with $D=2$ outperform the MTL baseline.
Again, PTA-F and PTA-FP perform best, and the mean improvement across all methods increases with $D$: 3.63\% for $D=2$; 4.07\% for $D=3$; and 4.37\% for $D=4$.
}
\end{figure}
All setups outperform the STL baseline, and all, except for PTA-I with $D=2$, outperform the MTL baseline.
Again, PTA-F and PTA-FP perform best, and the mean improvement across all methods increases with $D$.
The results show that although PTA implements behavior similar to MTL, when combined, their positive effects are complementary.
Finally, to test the scalability of these results, three diverse PTA methods with $D=4$ and $D=10$ were applied to the complete 50-task dataset: PTA-I, because it is the baseline PTA method; \mbox{PTA-F}, because it is simple and high-performing; and \mbox{PTA-HGD}, because it is the most different from PTA-F, but also relatively high-performing. The results are given in Table~\ref{tab:omni_results}.
\begin{table}
\centering
\caption{\label{tab:omni_results}\textbf{Omniglot 50-task results.}
Test error averaged across all tasks for each setup is shown.
Overall, the performance gains from MTL complement those from PTA, with PTA-F again the highest-performing and most robust method.}
\vspace{0.08in}
\footnotesize
\begin{tabular}{l c c c c} \toprule
Method & \multicolumn{2}{c}{Single-task Learning} & \multicolumn{2}{c}{Multitask Learning} \\ \midrule
Baseline & \multicolumn{2}{c}{35.49} & \multicolumn{2}{c}{29.02} \\ \midrule
 & $D=4$ & $D=10$ & $D=4$ & $D=10$ \\ \midrule
 PTA-I & 31.72 & 32.56 & 27.26 & 24.50 \\
 PTA-HGD & 31.63 & 30.39 & 25.77 & 26.55 \\
 PTA-F & \textbf{29.37} & \textbf{28.48} & \textbf{23.45} & \textbf{23.36} \\ \midrule
 PTA-Mean & 30.91 & 30.48 & 25.49 & 24.80 \\ \bottomrule
\end{tabular}
\end{table}
The results agree with the 20-task results, with all methods improving upon the baseline, and performance overall improving as $D$ is increased.

\subsection{IMDB Sentiment Analysis}

The experiments in this section apply PTA to LSTM models in the IMDB sentiment classification problem \cite{Maas:2011}.
The dataset consists of 50K natural-language movie reviews, 25K for training and 25K for testing.
There is a single binary classification task: whether a review is positive or negative.
As in previous work, 2500 of the training reviews are withheld for validation \cite{McCann:2017}.
The underlying model $\mathcal{F}$ is the off-the-shelf LSTM model for IMDB provided by Keras, with no parameters or preprocessing changed.
In particular, the vocabulary is capped at 20K words, the LSTM layer has 128 units and dropout rate 0.2,  and each meta-iteration consists of one epoch of training with Adam \cite{Kingma:14}.
This is not a state-of-the-art model, but it is a very different architecture from that used in Omniglot, and therefore serves to demonstrate the broad applicability of PTA.

The final three PTA methods from Section~\ref{subsec:omni} were evaluated with 4 and 10 decoders (Table~\ref{tab:imdb}).
\begin{table}
\centering 
\caption{\label{tab:imdb} \textbf{IMDB Results.}
All PTA methods outperform the LSTM baseline.
The best performance is achieved by PTA-HGD with $D=10$.
This method receives a substantial boost from increasing the number of decoders from 4 to 10, as the greedy algorithm gets to perform broader search.
On the other hand, PTA-I and PTA-F do not improve with the additional decoders, suggesting that, without careful control, too many decoders can overconstrain $\mathcal{F}$. 
}
\footnotesize
\vspace{0.08in}
\begin{tabular}{ l c c} \toprule
Method & \multicolumn{2}{c}{Test Accuracy \%} \\ \midrule 
LSTM Baseline ($D=1$) & \multicolumn{2}{c}{82.75 ($\pm 0.13$)} \\ \midrule
 & $D=4$ & $D=10$ \\ \midrule
PTA-I & 83.20 ($\pm 0.07$) & 83.02 ($\pm 0.11$) \\
PTA-HGD & 83.22 ($\pm 0.05$) & \textbf{83.51} ($\pm 0.08$) \\ 
PTA-F & \textbf{83.30} ($\pm 0.12$) & 83.30 ($\pm 0.08$) \\ \bottomrule
\end{tabular}
\end{table}
As in Section~\ref{subsec:omni}, all PTA methods outperform the baseline.
In this case, however, PTA-HGD with $D=10$ performs best.
Notably, PTA-I and PTA-F do not improve from $D=4$ to $D=10$, suggesting that underlying models have a critical point after which, without careful control, too many decoders can be overconstraining.
To contrast PTA with standard regularization, additional Baseline experiments were run with dropout rates $[0.3, 0.4, ..., 0.9]$. At 0.5 the best accuracy was achieved: 83.14 ($\pm 0.05$), which is less than all PTA variants except PTA-I with $D=10$, thus confirming that PTA adds value.
To help understand what each PTA method is actually doing, snapshots of decoder parameters taken every epoch are visualized  in Figure~\ref{fig:imdb_viz} with t-SNE \cite{Maaten:2008} using cosine distance.
\begin{figure}
\centering
\scriptsize
\hspace*{-20pt} \begin{tabular}{c c c}
\includegraphics[width=0.33\linewidth]{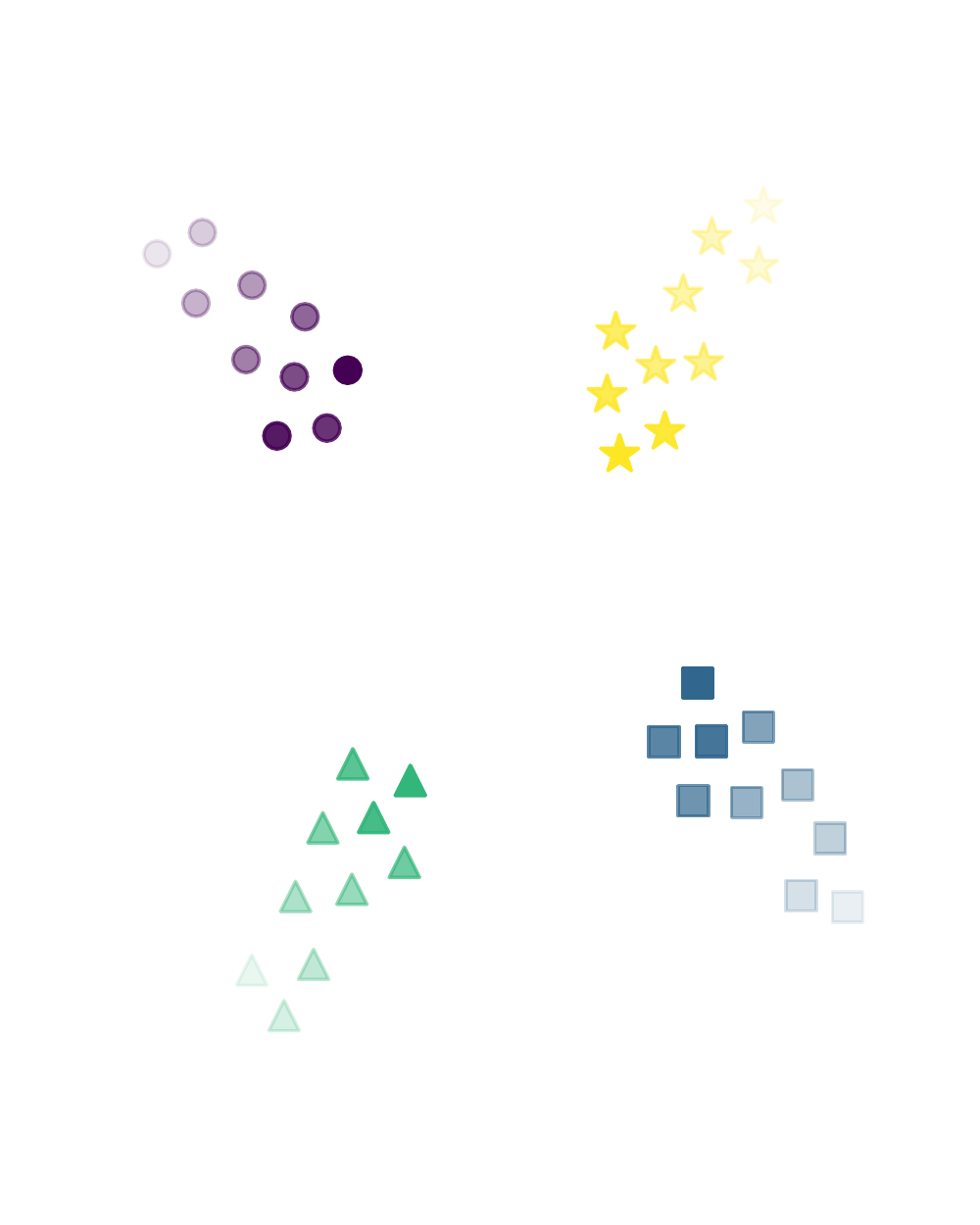} & \includegraphics[width=0.33\linewidth]{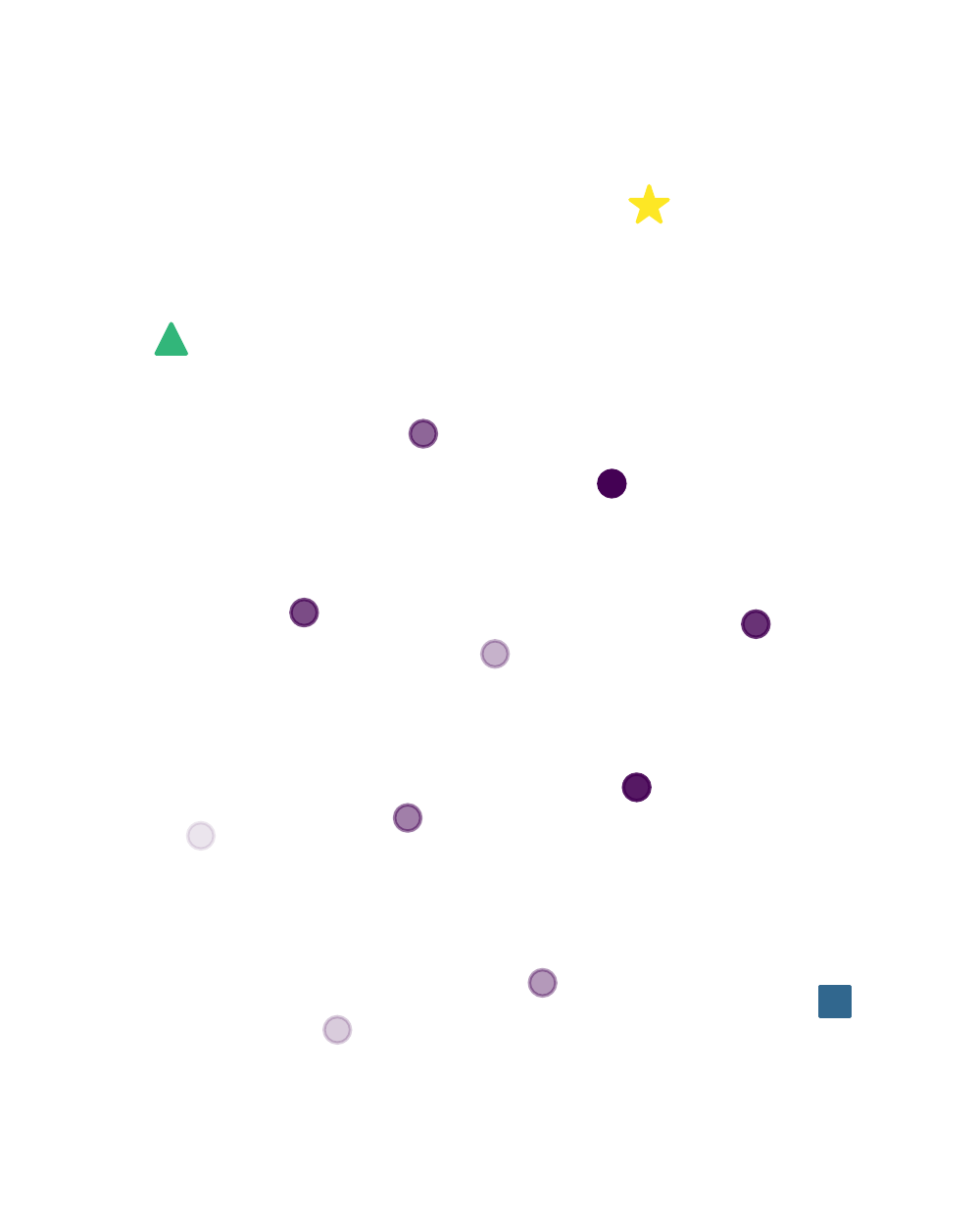} &
\includegraphics[width=0.33\linewidth]{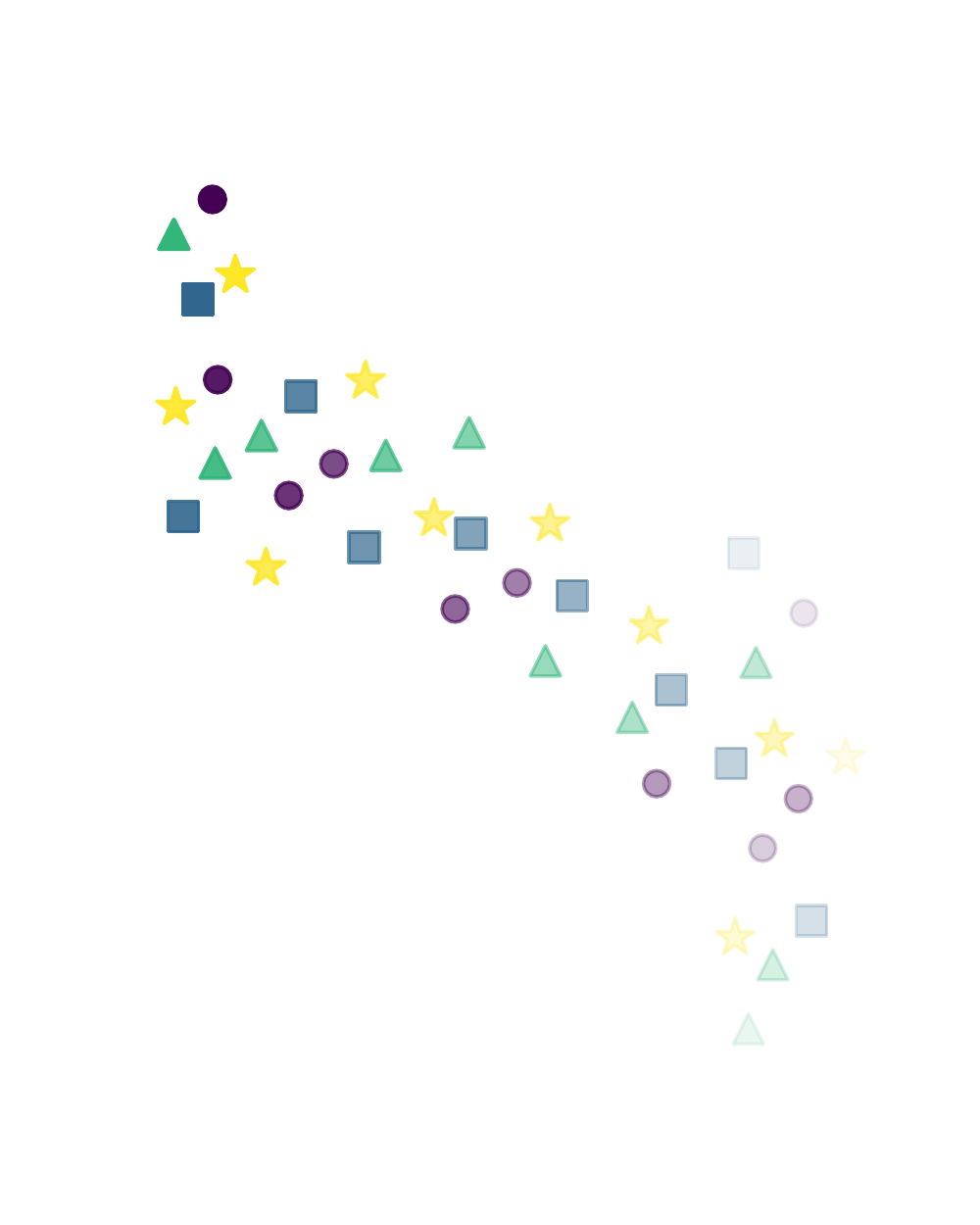}\\
(a) PTA-I & (b) PTA-F & (c) PTA-HGD \\
\end{tabular}
\caption{\label{fig:imdb_viz} \textbf{Pseudo-task Trajectories.}
t-SNE \cite{Maaten:2008} projections of pseudo-task trajectories, for runs of PTA-I, PTA-F, and PTA-HGD on IMDB.
Each shape corresponds to a particular decoder; each point is a projection of the length-129 weight vector at the end of an epoch, with opacity increasing by epoch.
The behavior matches our intuition for what should be happening in each case: (a) When decoders are only initialized independently, their pseudo-tasks gradually converge; (b) when all but one decoder is frozen, the unfrozen one settles between the others; (c) when a greedy method is used, decoders perform local exploration as they traverse the pseudo-task space together.
}
\end{figure}
The behavior matches our intuition for what should be happening in each case: When decoders are only initialized independently, their pseudo-tasks gradually converge; when all but one decoder is frozen, the unfrozen one settles between the others; and when a greedy method is used, decoders perform local exploration as they traverse the pseudo-task space together.

\subsection{CelebA Facial Attribute Recognition}

To further test applicability and scalability, PTA was evaluated on CelebA large-scale facial attribute recognition \cite{Liu:2015b}.
The dataset consists of $\approx$200K $178\times218$ color images.
Each image has binary labels for 40 facial attributes;
each attribute induces a binary classification task.
Facial attributes are related at a high level that deep models can exploit, making CelebA a popular deep MTL benchmark.
Thus, this experiment focuses on the MTL setting.

The underlying model was Inception-ResNet-v2 \cite{Szegedy:2016}, with weights initialized from training on ImageNet \cite{Russakovsky:2015}.
Due to computational constraints, only one PTA method was evaluated: \mbox{PTA-HGD} with $D=10$.
PTA-HGD was chosen because of its superior performance on IMDB, and because CelebA is a large-scale problem that may require extended pseudo-task exploration; Figure~\ref{fig:imdb_viz} shows how PTA-HGD may support such exploration above other methods.
Each meta-iteration consists of 250 gradient updates with batch size 32.
The optimizer schedule is co-opted from previous work \cite{Gunther:2017}: RMSprop is initialized with a learning rate of $10^{-4}$, which is decreased to $10^{-5}$ and $10^{-6}$ when the model converges.
PTA-HGD and the MTL baseline were each trained three times.
The computational overhead of PTA-HGD is marginal, since the underlying model has 54\mbox{M} parameters, while each decoder has only 1.5\mbox{K}.
Table~\ref{tab:celeba_results} shows the results.
\begin{table}
\centering
\caption{\label{tab:celeba_results} \textbf{CelebA results.}
Comparison of PTA against state-of-the-art methods for CelebA, with and without ensembling.
Test error is averaged across all attributes.
\mbox{PTA-HGD} outperforms all other methods, establishing a new state-of-the-art in this benchmark.
}
\vspace{0.08in}
{
\footnotesize
\begin{tabular}{l c} \toprule
MTL Method & \% Error \\ \midrule
Single Task \cite{He:2017} & 10.37 \\
MOON \cite{Rudd:2016} & 9.06 \\
Adaptive Sharing \cite{Lu:2016} & 8.74 \\
MCNN-AUX \cite{Hand:2017} & 8.71 \\
Soft Order \cite{Meyerson:2018} & 8.64 \\
VGG-16 MTL \cite{Lu:2016} & 8.56 \\
Adaptive Weighting \cite{He:2017} & 8.20 \\
AFFACT \cite{Gunther:2017} (best of 3) & 8.16 \\ \midrule
MTL Baseline (Ours; mean of 3) & 8.14 \\
PTA-HGD, $D=10$ (mean of 3) & \textbf{8.10} \\  \midrule
Ensemble of 3: AFFACT \cite{Gunther:2017} & 8.00 \\
Ensemble of 3: PTA-HGD, $D=10$ & \textbf{7.94} \\  \bottomrule
\end{tabular}
}
\end{table}
\mbox{PTA-HGD} outperforms all other methods, thus establishing a new state-of-the-art in CelebA.
Figure~\ref{fig:drop_rates} shows resulting dropout schedules for \mbox{PTA-HGD}.
\begin{figure}
\centering
\includegraphics[width=0.98\linewidth]{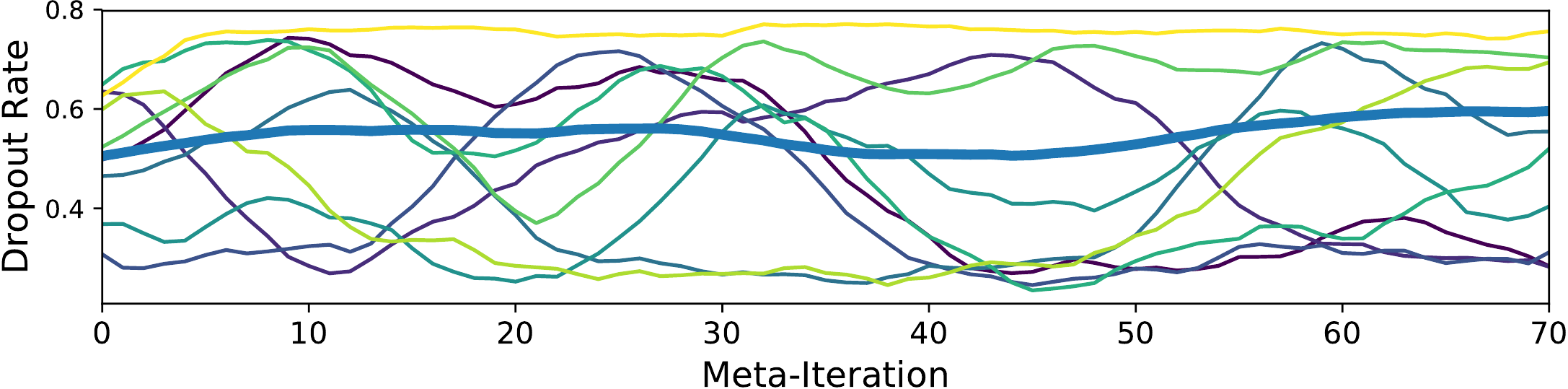}
\caption{\label{fig:drop_rates} \textbf{CelebA dropout schedules.} 
The thick blue line shows the mean dropout schedule across all 400 pseudo-tasks in a run of PTA-HGD.
Each of the remaining lines shows the schedule of a particular task, averaged across their 10 pseudo-tasks.
All lines are plotted with a simple moving average of length 10.
The diversity of schedules shows that the system is taking advantage of PTA-HGD's ability to adapt task-specific hyperparameter schedules.}
\end{figure}
No one type of schedule dominates; \mbox{PTA-HGD} gives each task the flexibility to adapt its own schedule via the performance of its pseudo-tasks.

\section{Discussion and Future Work}
\label{sec:discussion}

The experiments in this paper demonstrated that PTA is broadly applicable, and that it can boost performance in a variety of single-task and multitask problems.
Training with multiple decoders for a single task allows a broader set of models to be visited. 
If these decoders are diverse and perform well, then the shared structure has learned to solve the same problem in diverse ways, which is a hallmark of robust intelligence.
In the MTL setting, controlling each task's pseudo-tasks independently makes it possible to discover diverse task-specific learning dynamics (Figure~\ref{fig:drop_rates}).
Increasing the number of decoders can also increase the chance that pairs of decoders align well across tasks.

The crux of PTA is the method for controlling pseudo-task trajectories.
Experiments showed that the amount of improvement from PTA is dependent on the choice of control method.
Different methods exhibit highly structured but different behavior (Figure~\ref{fig:imdb_viz}).
The success of initial methods indicates that developing more sophisticated methods is a promising avenue of future work.
In particular, methods from Section~\ref{subsec:single_task_separate_models} can be co-opted to control pseudo-task trajectories more effectively.
Consider, for instance, the most involved method evaluated in this paper: PTA-HGD.
This online decoder search method could be replaced by methods that generate new models more intelligently \cite{Bergstra:2011, Snoek:2012, Miikkulainen:2017, Real:2017, Zoph:2017}.
Such methods will be especially useful in extending PTA beyond the linear case considered in this paper, to complex nonlinear decoders.
For example, since a set of decoders is being trained in parallel, it could be natural to use neural architecture search methods \cite{Miikkulainen:2017, Real:2017, Zoph:2017} to search for optimal decoder architectures.
While ensembling separate PTA models is useful (Table~\ref{tab:celeba_results}), in preliminary tests na\"ively ensembling decoders for evaluation (Eq.~\ref{eq:pseudotask_ensemble}) did not yield remarkable improvements over the single best (Eq.~\ref{eq:pseudotask_evaluation}).
In a further preliminary test with IMDB, when $\mathcal{F}$ was \emph{not} shared, PTA-I outperformed PTA-HGD and PTA-F, indicating 
  that the latter two methods address dynamics that arise in joint training but not na\"ive ensemble training.
Developing PTA training methods for generating a more complementary set of decoders, coupled with effective methods for ensembling this set, could push performance even further, especially when decoders are more complex.

\section{Conclusion}
\label{sec:conclusion}

This paper has introduced \emph{pseudo-task augmentation}, a method that makes it possible to apply ideas from deep MTL to single-task learning.
By training shared structure to solve the same task in multiple ways, pseudo-task augmentation simulates training with multiple closely-related tasks, yielding performance improvements similar to those in MTL.
However, the methods are complementary: combining pseudo-task augmentation with MTL results in further performance gains.
Broadly applicable, pseudo-task augmentation is thus a promising method for improving deep learning performance.
Overall, this paper has taken first steps towards a future class of efficient model search algorithms that exploit intratask parameter sharing.

\section*{Acknowledgements}
We would like to thank Xin Qiu, Antoine Saliou, and the reviewers for providing valuable feedback that helped to solidify this work.

{\small
\bibliography{meyerson-icml2018}
\bibliographystyle{icml2018_style/icml2018}
}

\end{document}